\documentclass[journal,twoside]{IEEEtran}
\usepackage[T1]{fontenc}
\usepackage{wrapfig}
\usepackage{graphicx}
\usepackage{amsmath,amssymb}
\usepackage{amsthm}
\usepackage[ruled,vlined,linesnumbered]{algorithm2e}
\usepackage{subcaption}
\usepackage{siunitx}
\usepackage{booktabs}
\usepackage{color}
\usepackage{hyperref}

\newtheorem{proposition}{Proposition}
\newtheorem{corollary}{Corollary}

\theoremstyle{definition}
\newtheorem{definition}{Definition}
\theoremstyle{remark}
\newtheorem{remark}{Remark}

\newcommand{\fix}[1]{#1}
\DeclareMathOperator*{\argmax}{arg\,max}
\newcommand{\norm}[1]{\left\lVert #1 \right\rVert}

\begin{document}

\title{Learning Ordinal Response Policies in Rank-Based Stochastic Prize-Collecting Games}

\author{Malintha~Fernando,~Petter~\"{O}gren,~and~Silun~Zhang%
\IEEEcompsocitemizethanks{
  \IEEEcompsocthanksitem M. Fernando, P. \"{O}gren, and S. Zhang are with KTH Royal Institute of Technology, Stockholm, Sweden.
  E-mail: \{malintha, petter, silunz\}@kth.se
}}


\maketitle

\begin{abstract}
The Team Orienteering Problem (TOP) generalizes many real-world multi-agent scheduling and routing tasks that occur in autonomous mobility, aerial logistics, and surveillance applications. While many flavors of the TOP exist for planning in multi-agent systems, they assume that all the agents cooperate toward a single objective; therefore, they do not extend to settings when they compete in reward-scarce environments. We propose Stochastic Prize-Collecting Orienteering Games (SPCOG) as an extension of the TOP to plan in the presence of self-interested agents operating on a graph, under energy constraints and stochastic transitions. A theoretical discussion on complete and star graphs establishes that there is a unique pure Nash equilibrium in SPCOGs that coincides with the optimal routing solution of an equivalent TOP under rank-based conflict resolution. We propose the concept of Ordinal Rank (OR) as a concise representation of an agents' global rank and its location within a topological, well-defined neighborhood. Empirical evaluations conducted on real-world, road-network graphs under both dynamic and stationary prize distributions show that in parameter-sharing settings, the policies that leverage local information can outperform those policies leverage global information when the former is conditioned on the OR rather than the global rank, indicating that the OR acts as a strong inductive bias in multi-agent games on graphs. The OR-conditioned policies also generalize much better to games with large number of agents compared to global-rank conditioned policies. Finally, we propose Fictitious Ordinal Response Learning (FORL) as an entropy-regulated algorithm to obtain convergent policies in independent-learning settings in prize-collecting games on graphs. 
\end{abstract}


\section{Introduction}

The team orienteering problem (TOP) generalizes the classical traveling salesman problem (TSP) to compute a set of ``tours'' that visit a \textit{subset} of cities while maximizing a \textit{collective utility} under a given travel budget~\cite{chao1996team}.
It provides a unified framework for an array of real-world applications, such as logistic networks, and path planning~\cite{9636854,8206597}.
While several TOP variants exist, e.g., TOPs with task deadlines~\cite{gunawan_orienteering_2016}, they share the notion that all robots must maximize the collective return, thus failing to model settings with heterogeneous or competing objectives, e.g., competing individuals or teams in resource-scarce environments.
To address this gap, we propose a game-theoretic variant of the TOP, played by self-interested robots (agents), which we refer to as the \emph{Stochastic Prize-Collecting Orienteering Games (SPCOG)}.

In multi-robot systems, significant research is underway to design robotic teams to automate last-mile delivery and hotspot monitoring, e.g., in surveillance and ocean monitoring.
However, maintaining a large fleet of mobile robots is often impractical for service providers, who may instead delegate tasks to multiple externally-owned fleets with diverse capabilities, i.e.,
carrying capacity and sensors.
While the fleets can route themselves to service the depots \textit{independently}, they may not coordinate their plans with the other fleets due to inherent misalignments in mode of operation, capabilities, and economic incentives.
Therefore, conventional TOP-based routing solutions, which assume cooperative robots with a shared objective, are not directly applicable to systems of heterogeneous and self-interested fleets, motivating a TOP variant that accommodates individual objectives and strategic interactions.

The SPCOG also effectively capture the bi-level, profit-driven nature of crowd-sourced applications where rational operators, e.g., delivery drivers, compete to maximize their individual profits among other operators, while the service providers aim to maximize their own profits while competing with similar other providers, e.g., \textit{Uber}, and \textit{Lyft}.

Compared to the TOP, given a pre-announced \textit{conflict-resolution} rule, agents in an SPCOG aim to maximize their individual rewards rather than the collective utility.
A conflict-resolution rule is a rule that defines how to distribute the prize when multiple agents serve the same node.
Additionally, in the SPCOG formulation, the start and the terminal nodes need not be fixed; there may exist multiple terminal nodes in an SPCOG, allowing the agents to choose according to their preferences.
We show that when the graph is complete, there exists a unique pure Nash equilibrium (PNE) that can be obtained under local information and rank-based conflict resolution.
We propose the concept of ``Ordinal Rank" --an agent's effective rank in temporarily-formed local neighborhoods, for finding \textit{localized best responses} at each stage under rank-based conflict resolution rules in games on graphs.
Secondly, we propose Fictitious Ordinal Response Learning (FORL) --an entropy-regulated scheduling algorithm inspired by the classical fictitious-play method to obtain convergent policies in independent-multi-agent reinforcement learning (MARL) settings.
In experiments conducted on real-world graphs, we observe that when the robots' local observations are conditioned on their ordinal ranks, policies learned under parameter-sharing achieve 95\% total team reward compared to that of optimal routing solution in an equivalent TOP, show better generalizability to previously unseen prize distributions, and scale better with the number of agents.

\begin{figure}[t]
  \centering
  \includegraphics[width=1.03\columnwidth]{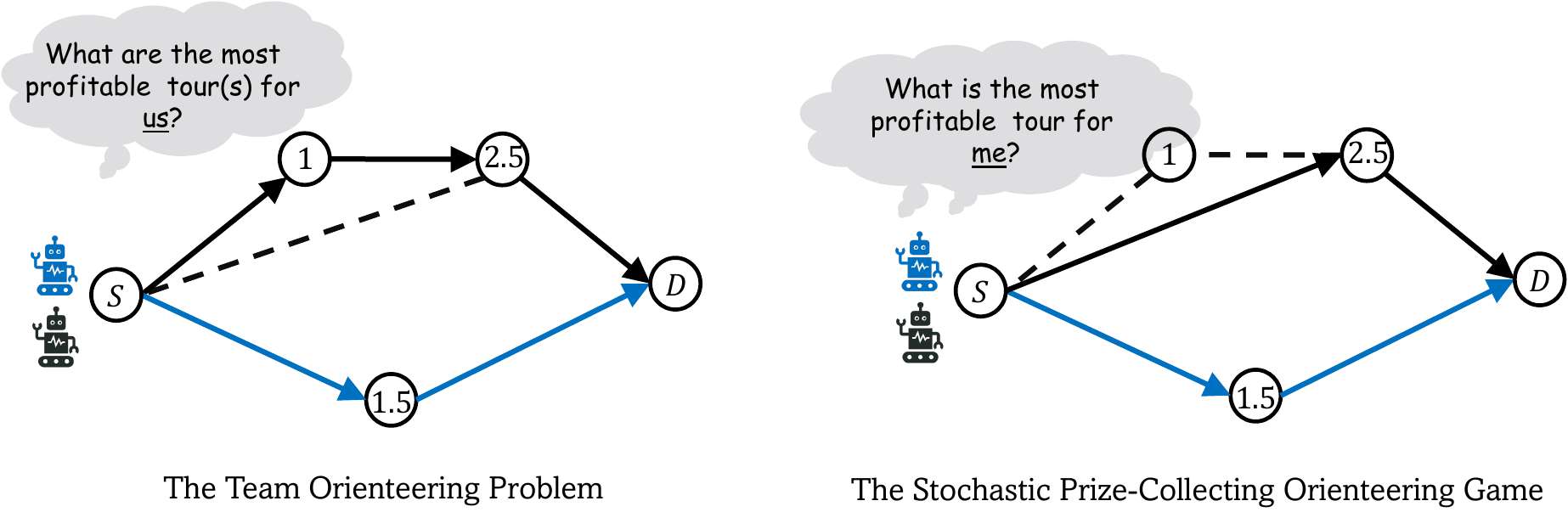}
  \caption{\small{A high-level comparison between the team orienteering problem (TOP) (left) and the proposed stochastic prize-collecting orienteering game (SPCOG) (right).}}
  \label{fig:cover}
\end{figure}

Fig.~\ref{fig:cover} compares best-response strategies in a prize-collecting game to the optimal strategy of an equivalent TOP.
Here, $S, D$ denote the start and terminal nodes. One robot gets priority during conflicts. In the TOP, the optimal policy (left) routes the robots to maximize the \emph{total reward} regardless of their ranks. However, when the robots are \emph{self-interested}, the senior robot must directly move to prize (2.5) ignoring the prize (1) to prevent being cut-in-front, while the other must move to (1.5) (right).
Therefore, the maximum total reward in the SPCOG (4) is less than that of the TOP's (5).

The key contributions of this work are as follows.
\begin{enumerate}
    \item We introduce SPCOGs for decision-making in competitive, stochastic routing environments under energy constraints, and discuss the existence of a unique PNE on different graphs.
    \item We propose the concept of Ordinal Rank (OR) that effectively summarizes the information about one's most influential neighbors in rank-based games on graphs to find best-responses locally.
    \item We propose Fictitious Ordinal Response Learning (FORL), an entropy-regulated scheduling algorithm to learn convergent, independent policies in routing games on graphs.
    \item Finally, we provide extensive experiments to evaluate the scalability, generalizability of the learned policies in real-world road networks.
\end{enumerate}

\subsection{Related Literature}

\subsubsection{The Team Orienteering Problem}
The TOP is a multi-agent generalization of the classic traveling salesman problem (TSP) which does not necessitate the salesmen to visit all the cities, rather a subset that maximizes the \textit{collective} reward given a maximum travel budget~\cite{gunawan_orienteering_2016}.
In~\cite{park_learn_nodate}, authors formulate the multiple TSP as a Markov decision process, where the salesmen contribute toward a common global objective.
The TOP has been extensively studied in robotics for scheduling in stochastic environments.
In~\cite{8206597}, authors consider a single agent orienteering problem with stochastic prizes, while in~\cite{9636854}, authors formulate the multi-robot ocean monitoring under the TOP.
Multi-player games on graphs have a relatively short literature compared to matrix-form games \cite{bouyer2019computation}, and the game-theoretic variants of the TOP have only recently gained attention in operational research.
In~\cite{alvarez-miranda_competing_2024}, \fix{authors} introduce a two-agent game where a leader seeks to interdict the follower's revenue in pure competitive settings.
A related game-theoretic formulation is considered in~\cite{varakantham_direct_nodate} where the agents experience a cost proportionate to the number of agents at a chosen node.
However, these formulations restrict the agents to observe the same cost during conflicts, and do not generalize to exogenous processes.
In~\cite{murray2020prize}, authors studied the multi-robot prize-collecting game on a graph. However, the agents must move one-after-the-other in their Stackelberg formulation as opposed to simultaneous moves in ours.

\subsubsection{Game-Theoretic Task Assignment}
Several recent works propose MARL-based approaches for multi-robot task assignment (MRTA) in dynamic, stochastic environments in cooperative~\cite{gui2024collaborative,yu2021optimizing,shen2023dynamic}, and competitive settings~\cite{Fernando2023GraphAM} for manufacturing, transportation and emergency-rescue applications.
Market auctions have also been studied as MRTA at the presence of self-interested robotic agents where individuals can propose task bundles they would like to execute by minimizing an individual cost~\cite{choi2009consensus}.
Oftentimes the auction-based methods require a central arbitrator who evaluates the robots' proposals.
This requires iterating through a number of task bundles that is exponential in the number of nodes in the graph --requiring extensive computation, while overlooking the stochastic environmental transitions.

Finally, we draw parallels to resource-selection games (RSG) in operations research~\cite{gkatzelis2016optimal}.
We highlight that the SPCOGs differ from RSGs in two crucial ways: firstly, the action space of an agent in SPCOGs is constrained by the graph structure, and secondly, SPCOGs are an extensive-form game as opposed to the RSG's normal-form nature.

\section{Preliminaries}

\subsection{Stochastic Games}
A stochastic game (SG) is a tuple $\langle \mathcal{N},\mathcal{S},\mathbb{T},\fix{R}, \fix{A}, \gamma \rangle$, where $\mathcal{N}= \{1,\dots,n\}$, a set of agents, $\mathcal{S}$ is the state space.
\fix{The joint action space} $A = \times_{i \in \fix{n}}A_i$ where $A_i$ is the action space of the $i$-th agent, $R = \{R_1, \dots, R_n\}$, the collection of the agents' reward functions such that $R_i: \mathcal{S} \times A \times \mathcal{S} \rightarrow \mathbb R, \forall i \in \mathcal{N}$.
The transition kernel $\mathbb{T}: \mathcal{S} \times A \rightarrow \Delta( S )$ defines the probability distribution over $S$ given the current joint state and the joint action.
Finally, $\gamma$ is the discounting factor in $i$'s value function $V_i^{\pi}(s^t)$ that is the \textit{total expected discounted cumulative reward} at any state $s^t \in S$ under a joint strategy $\pi = \pi_i \times \pi_{-i}$.
Here, we use $-i$ to denote the set of all other agents except $i$.
The objective of an agent $i$ in an SG is to find the strategy $\pi^*_i$ that maximizes its value function as
\begin{equation}
    \pi_i^* \gets \argmax_{\pi_i (a_i^t| s^t)} V_i^{\pi_i,\pi_{-i} } (s^t)
    \label{eq:obj}
\end{equation}
for all $t$ using its local observations.

\begin{definition}[Nash Equilibrium]
In a stochastic game, a joint strategy $\pi$ is called a Pure Nash Equilibrium (PNE) if no player has an incentive to unilaterally deviate from it, i.e.,
    \begin{equation}
        V^{\pi_i, \pi_{-i}}_i(s^t) \geq V^{\pi'_i, \pi_{-i}}_i(s^t)
    \end{equation}
    for all $i \in \mathcal{N}$, $s^t \in S$ and $\pi'_i \neq \pi_i$~\cite{hu2003nash}.
\end{definition}
A \textit{stage game} is a \textit{one-shot} game played by the agents during a single timestep of the stochastic game.

\section{Stochastic Prize-Collecting Orienteering Game}

Let  $\mathcal{G}$ = $(\mathcal{V}, \mathcal{E}, \mathcal W)$ define a weighted-undirected graph, where $\mathcal{V}$ = $\{1,\dots,V\}$, $\mathcal{E} \subseteq \{(u,v): u,v \in \mathcal{V}\}$ are the sets of nodes and edges.
Here, $\mathcal{W}$ is a symmetric matrix that defines the travel cost over an edge, \fix{without loss of generality, we assume} $\mathcal{W}(u,v) \fix{\in [0,w_{\max}]}$, for any $(u, v) \in \mathcal{E}$.
Let $p_u^t(\beta_u) \in \mathbb{R}_{\ge 0}$ denote the prize state at time $t$ for each node $u \in \mathcal{V}$, where $p_u^t(\beta_u)$ is a random variable drawn from distribution $w_u(\beta_u)$ with $\beta_u$ being some sufficient statistics for the value distribution. We define the prizes vector $\mathbf{p}^t(\beta)= \big(p^t_u(\beta_u) \big)_{u\in \mathcal V}$.

An agent $i$'s state space is $\mathcal{S}_i = \mathcal{V} \times \mathbb{R}_{\geq 0}^2$, containing the tuples $\langle x_i^t, l_i^t, i\rangle$, whose elements are $x_i^t \in \mathcal{V}$ that is $i$'s location at $t$, $l_i^t \in [0, L_{\max}]$ which is $i$'s remaining travel budget under a common maximum budget $L_{\max}$ and its index $i$.
Each agent has a starting node $s_i$ and a destination node $d_i$ in $\mathcal{V}$.
We consider $d_i = d, \forall i$ in this work.
The terminal reward $p_d$ at $d$ is fixed such that $p_d \gg p_u, \forall u \neq d$.

Next, we extend the definition of an SG by including $\mathcal{G}$ and a conflict-resolution rule $\Xi$ to define an SPCOG, \fix{denoted by} $\Gamma$.
The conflict-resolution rule $\Xi$ defines a prize-sharing criteria if multiple agents visit the same node, simultaneously.

\fix{An SPCOG is represented as}, $\Gamma$ = $\langle \mathcal{G}, \mathcal{N}, \mathcal{S}, \mathbb{T}, R, A, \gamma, \Xi \rangle$,
where $\mathcal{S} = \times_{i \in N} \mathcal{S}_i \times \mathbb{R}^{|\mathcal{V}|}$ which include the states of the agents and the prizes.
For brevity, we ignore the time index and define the action space of an agent as $A_i = \mathcal{V}$. \fix{Note that}, the actions available for $i$ at $t$ are the set of $x_i^t$'s \textit{one-hop} neighbors.
We refer to this set as $i$'s \textit{reachable set} at $t$ such that $S_i^t = \{v \in \mathcal{V}: (x_i^t, v) \in \mathcal{E}\}$.

We focus on a rank-based conflict-resolution rule that allocates the prize at a node wholly to the agent with the smallest index (highest global rank).
Let the index numbering reflect the \textit{global rank} of the agents.
Then the \fix{highest} prize value an agent $i$ receives by visiting a node $v$ at $t$ is $p_v^t$.
The agent may collect it fully given that either it has not already been collected by another agent at a previous timestep, or no agent \fix{with rank} $j$ ($< i$) visits $v$ at $t$.
Formally, for some $i$ whose $x_i^t = v$,
\begin{equation}
    R^{t+1}_i = \begin{cases}
        0 &\text{ if } \exists \text{ } j  \in \mathcal{N}, j < i, \text{ such that } x^{t}_j = v, \\
        p^t_v &\text{ otherwise}.
    \end{cases}
    \label{eq:rew}
\end{equation}
Once collected, the prize gets set to 0, and re-drawn as below.
\begin{equation*}
    p^{t+1}_v = \begin{cases}
        \sim \omega_u(\beta_u) & \text{ if }  p^t = 0 \land \exists i \in \mathcal{N}\text{ such that } x^t_i = v, \\
        p^{t}_v &\text{ otherwise},
    \end{cases}
\end{equation*}

Finally, we define $\mathbb{I}^t_i$ to indicate the availability of an agent $i$ at $t$, such that $\mathbb{I}^{t+1}_i=1$ if $l_i^t > 0 \land x_i^t \neq d$, or zero otherwise.
Then we can write the value function of an agent at $t=0$ as
\begin{equation}
    V_i^{\pi} (s^0) = \mathbb{E}_{\substack{a_i \sim \pi_i \\ a_{-i}\sim \pi_{-i}}} \Big[ \sum_{t=0}^{T} \mathbb{I}_i^t \gamma^t R_i^t(a_i^t, a_{-i}^t, s^t) \mid s^{0} \Big]
    \label{eq:ag_ob1}
\end{equation}
where, $s^0 \in \mathcal{S}$ is the initial state.

\section{Existence of Equilibria on Different Graphs}
The state space of an SPCOG is exponential in the number of agents as one must find a best response to the others' strategies, while a single agent's state space is factorial in the number of nodes.
In~\cite{murray2020prize}, authors proposed a game where the agents reserve their full routes between $s$ and $d$ in a leader-follower manner.
The \textit{Stackelberg} response of an agent in \fix{leader-follower formulation} is to choose the \textit{optimal} route during its turn by excluding the prizes on previously reserved routes by the senior agents.
In contrast, SPCOG is a simultaneous-move game in which no agent can reserve a tour; meaning, any agent can intercept another agent's pre-determined tour during the current or a future stage of the game.
This induces uncertainty about the agents' actions causing the game significantly more challenging to solve compared to the Stackelberg variant.

According to Nash's existence theorem, there is at least one Nash equilibrium in the mixed-strategies in an SPCOG, since the number of agents and the action space is finite.
However, a more interesting question is whether an SPCOG played on an arbitrary graph possesses a pure Nash equilibrium (PNE)?
We will show that the answer to this question depends on the structure of the graph and the prize distribution.
Here we only focus on a single realization of prizes; $\mathbf{p}^0(\omega_u) \sim \times_{u \in \mathcal{N}} \omega_u$, without re-population.

\begin{proposition}\label{thm:complete}
    There is a unique pure strategy Nash equilibrium for all $\mathbf{p}^0(\omega_u), \forall \omega_u$ when $\mathcal{G}$ is complete, and a fixed rule is available for breaking ties between multiple maximal rewarding strategies for an agent.
\end{proposition}
\begin{proof}
    Consider some agent $i \in \mathcal{N}$, $|\mathcal{V}|>n$ and the descending order of prizes $\mathbf{p}_{desc} = (p_1, \dots, p_{u}, p_{u+1}, \dots \fix{p_{|\mathcal{V}|}})$, where $p_u \geq p_{u+1}, \forall u \in \mathcal{N}$.
    The best action for $i$ during the first stage game is to choose the $i$-th largest prize from $\mathbf{p}_{desc}$.
    Because, any deviation made by $i$ to left in $\mathbf{p}_{desc}$, (e.g., $i-1$-th prize), will result in zero rewards since there is always a senior agent, ($i-1$), who can override $i$.
    Any deviations to right will strictly reduce the reward.
    Therefore no agent can do better than choosing the prize that matches their rank.
    In case multiple maximal prizes are available, one may invoke a tie-breaking rule that is known to everyone, e.g., \textit{agent with the lower rank choosing the prize with the minimum travel cost}.
    Then, re-sort $\mathbf{p}_{desc}$, and repeat the same procedure in the subsequent stages.
    If the $i$-th prize is 0, or $l_i^t$ is not sufficient to reach another prize, then move to $d$.
    This concludes the pure Nash equilibrium strategy for any $i$ for any $\mathbf{p}^0$.
    The uniqueness of the solution is guaranteed by the uniqueness of $\mathbf{p}_{desc}$ if no prize repeats, and by the tie-breaking rule, otherwise.
\end{proof}
\begin{remark}\label{rem:comp}
It is clear from Proposition~\ref{thm:complete} that the PNE \fix{under a fixed tie-breaking rule} also results in the PNEs in the stage games. Also, the total team reward under the PNE is equal to that of the optimal solution obtained by solving an equivalent TOP for a given $L$.
\end{remark}
Next, consider $\mathcal{G}_s$ is a star graph, i.e., edge set $\mathcal E_s=\{ (w,u) : u\in \mathcal V, u\neq w  \}$,
where $w$ is some \textit{staging node} that does not contain a prize.
Then Proposition~\ref{thm:complete} also applies for $G_s$: the PNE strategy for the $i$-th agent is moving to the node with the $i$-th largest prize in every alternating stage.
\begin{remark}\label{rem:partial}
It is sufficient for an agent to only know its rank, the number of players, and the graph state to compute the PNE on a complete graph.
\end{remark}

\begin{proposition}\label{thm:no_fip}
    There exists some joint distribution of prizes $\times_{u \in \mathcal{V}}\omega_u(\bar{p}_u)$ and a graph $\mathcal{G}$, for which an SPCOG played on an arbitrary $\mathcal{G}$ does not possess a PNE.
\end{proposition}
\begin{proof}
    This is a proof by \fix{counter-example}.
    Consider the SPCOG shown in figure~\ref{fig:fip}, played on a graph $\mathcal{G}$.
    Consider two players $i, i+1$, referred to as \textit{senior} and \textit{junior}, are at $s$, and share the same destination $d_1 = d_2 = d$.
    Consider the pure strategies available to each player: $\pi^1 = (1,2)$, $\pi^2 = (2)$, $\pi^3 = (3)$ (we use the superscripts to avoid confusing with agents' strategies $\pi_i$).
    \fix{Here, a strategy $(u,v), u,v \in \mathcal{V}$ refers to the path from $s$ to $d$ through the nodes in the strategy, i.e., the path corresponding to strategy $(1,2)$ is $s,1,2,d$.}
    Given that $L_{\max} = 3$, no strategy that contains a moving-back action will terminate at $d$, and is not profitable. Therefore, we exclude all such strategies from the payoff matrix (~\ref{fig:fip} (Right)).
    The rows indicate the strategies of the senior player.
    Since all the strategies contain the terminal reward $p_d$, we have also omitted it.
    Now consider the strategy pairs that contain the best responses of the junior player for any strategy of the senior player: $(\pi^1, \pi^2), (\pi^2, \pi^3), (\pi^3, \pi^1)$. Since no strategy of the senior player in them is a best response to the junior player's strategy, this game does not contain a Nash equilibrium in the pure strategies for any $\alpha \in (0,1)$.
\end{proof}
In~\cite{murray2020prize}, authors show that the leader-follower route-selection game has a Stackelberg equilibrium for any $\mathcal{G}$ as long as a fixed tie-breaking rule is imposed to handle multiple maximal strategies.
However, Proposition~\ref{thm:no_fip} shows that \fix{such a} result does not necessarily carry \fix{in SPCOGs} when agents move simultaneously regardless the graph is undirected or directed.

\begin{figure}[t]
  \centering
    \centering
    \includegraphics[trim={0cm 1cm 0cm 1cm},clip,
                     width=0.8\linewidth]{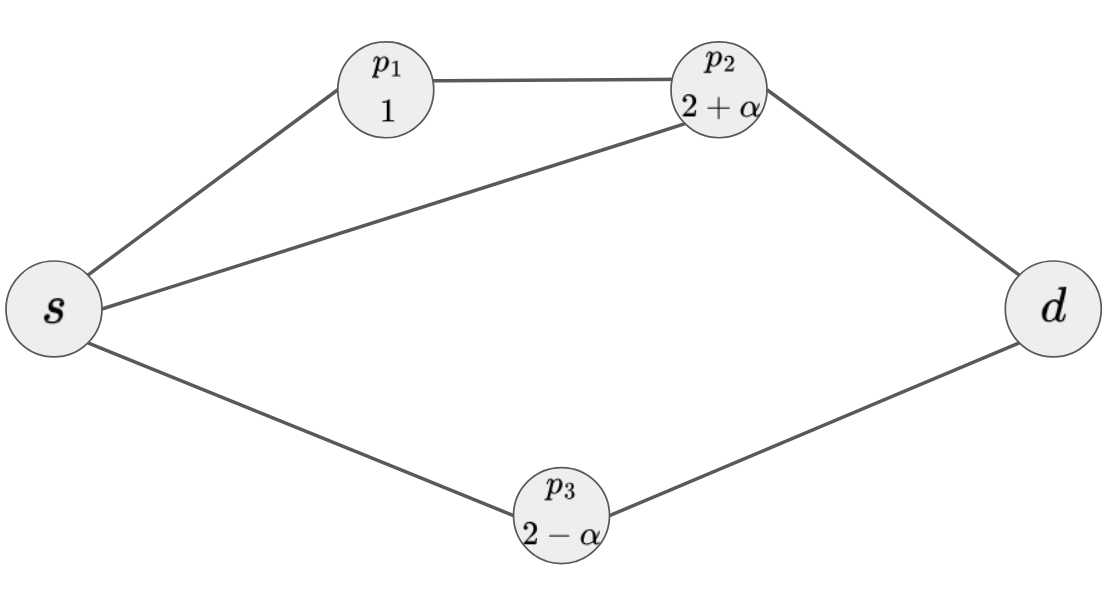}
  
  \caption{\small A two-agent SPCOG on $\mathcal{V}=\{s,1,2,3,d\}$,
    $\mathcal{W}(u,v)=1$, $L_{\max}=3$, $p_s=0$, $p_d\gg p_u\,\forall u\neq d$.}
  \label{fig:fip}
\end{figure}    

\begin{table}[t!]
    \setlength{\tabcolsep}{4pt}
    \small
    \centering
    \begin{tabular}{r|ccc}
      \toprule
               & $\pi^1$                     & $\pi^2$                     & $\pi^3$                     \\
      \midrule
      $\pi^1$  & $(3{+}\alpha,\,0)$          & $(1,\,2{+}\alpha)$          & $(3{+}\alpha,\,2{-}\alpha)$ \\
      $\pi^2$  & $(2{+}\alpha,\,1)$          & $(2{+}\alpha,\,0)$          & $(2{+}\alpha,\,2{-}\alpha)$ \\
      $\pi^3$  & $(2{-}\alpha,\,3{+}\alpha)$ & $(2{-}\alpha,\,2{+}\alpha)$ & $(2{-}\alpha,\,0)$          \\
      \bottomrule
    \end{tabular}
    \caption{Payoff matrix; rows are senior, columns are junior strategies. No PNE exists for any $\alpha\in(0,1)$.}
    \label{tab:fip}
\end{table}

\section{Approach}
Motivated by Remark~\ref{rem:partial} and \ref{thm:no_fip}, we propose a method to drive learning strategies toward an mixed-strategy NE in incomplete graphs, especially, when the full-state is not observable.
First, we introduce the concept of \textit{Immediate opponents} to isolate the most-influential opponents of an agent playing a routing game of a graph.
Then, we propose ordinal rank based on the immediate opponents to find best responses for myopic agents, and propose an entropy-regulated fictitious play routine to find mixed-strategy NE in SPCOGs.

\subsection{Finding Immediate Opponents}
Consider the reachable sets $S_1, \dots, S_n$ (ignoring the time index since we focus on a single stage).
First, we define a secondary graph whose nodes are the set of agents, and the edges represent the mutuality of their feasible actions.
We refer to this graph as the ``separating graph" denoted by $\bar{\mathcal G}$.
The connected components of this graph represent a stage game that is localized to a subset of agents.
The formal definition follows below.
\begin{definition}[Separating Graph]
     The \textit{separating graph} is defined as $\Bar{\mathcal{G}} = (\Bar{\mathcal{N}}, \Bar{\mathcal{E}})$ whose set of nodes $\Bar{\mathcal{N}} = \mathcal{N}$, and the edges $\Bar{\mathcal{E}} = \{(i,j) : S_i \cap S_j \neq \emptyset, \forall i, j \in \Bar{\mathcal{N}}\}$.
\end{definition}
The set of \textit{connected components} in $\Bar{\mathcal{G}}$ is $\Lambda = \{\lambda_y :\lambda_y \subseteq \Bar{\mathcal{N}}, y = 1,2,\dots \}$ where each $\lambda_y$ is the largest connected component such that there is no path between any node in $\lambda_y$ to any node in $\lambda_z$, for any $\lambda_y, \lambda_z \in \Lambda, y \neq z$.
Further, for every $i \in \mathcal{N}$, define a mapping $\lambda(i) \in \Lambda$ such that $i \in \lambda(i)$.
From the definition of $\Lambda$, $\lambda(i)$ is well-defined.

There are efficient algorithms to compute the connected components $\Lambda^t$ of a finite graph rather straightforwardly in linear time, e.g., \cite{hopcroft1973algorithm}.
One can iterate through all the nodes in $\bar {\mathcal{G}}^t$, and launch a depth-first traversal at each node to find its neighbors, then insert them into a single connected component $\lambda_y^t$.
Once all the neighbors are visited, insert $\lambda_y^t$ into $\Lambda^t$.
Then, finding $\lambda(i,t)$ is simply a matter of finding the component of $i$.

\begin{definition}[Localized-Stage Game]
A subset of agents are participating in a localized stage game $y$ if and only if they belong to the same connected component $\lambda_y$.
\end{definition}
In other words, the actions of any agent in the subset $\lambda(i)$, except $i$ has an immediate impact on $i$'s actions.
Therefore, we refer to this subset of agents except $i$ as $i$'s ``immediate opponents''.
Finally, we introduce the concept of \textit{ordinal rank} (OR): a metric that summarizes positional, and rank information the immediate opponents.

\begin{definition}[Ordinal Rank]
For each agent $i \in \mathcal{N}$, we define its \textit{ordinal rank} as
\begin{equation}
\mathcal{I}_i = 1 + | \{j : j \in \lambda(i), j < i\}|.
\label{def:or}
\end{equation}
\end{definition}

By only communicating the OR information to an agent even without the explicit positional, rank, or reachability information of its opponents, it can choose best-response actions to their immediate opponents within a localized stage game.
When $\mathcal{G}$ is complete, the separating graph has only one connected component at any stage of the SPCOG, and thus, one's ordinal rank is equivalent to its global rank.

In the experiments, we observed that the strategies trained on local information conditioned on the OR significantly outperform those trained on the global state, or conditioned on the global rank.

Assuming a myopic agent is only interested in best-responding to its immediate opponents, it is sufficient for one to know information about its localized game, and \fix{the prizes} $\mathbf{p}^t$, rendering SPCOGs solvable with partial observations.
Thus, let $o_i^t \in \mathcal{O}_i$ define a local observation vector obtained from the mapping $\mathcal{O}_i: \mathcal{S} \rightarrow O_i$, where by $O_i$ \fix{we denote} $i$'s \textit{observation space}.
In our experiments, we show that one can in fact obtain reasonable best-response policies by using local information conditioned on OR, --implicitly indicating best-responding to localized stage games in fact lead to reasonable best-responses in the Markov SPCOG.
In this section, we will use the term ``policy'' instead of ``strategy'' to better align with the RL literature.

\subsection{Fictitious Ordinal Response Learning}
This subsection proposes an efficient algorithm for learning stationary policies in SPCOGs played on arbitrary graphs under independent learning paradigm.
In the proposed method, we sequentially update the agents' policies following their ranks one-at-a-time, assuming the others' policies remain stationary between updates --a concept central to the \textit{fictitious play} in classical game theory~\cite{shoham_multiagent_nodate}.
We will empirically show that this alternating fictitious play process will lead to best response policies against one's senior opponents when $\mathcal{G}$ is incomplete, and also yield PNE policies when $\mathcal{G}$ is complete.

Algorithm~\ref{algo:ec-brpl} outlines the proposed method.
During the initial stage (\textit{bootstrapping} stage), we only update the highest-ranked agent's policy, while the others follow the random policy (parameterized by $\theta^0_.$). \fix{This stems from the intuition that} any agent must first learn to solve the underlying single-agent orienteering problem \fix{under the environmental stochasticity which incorporates the senior agents' unknown actions, and the distribution of the prizes}.
\fix{The confidence of a policy is measured by its entropy that is, a policy with high entropy has a low confidence about its actions, and vice-versa.}
Once the highest-ranked agent's policy learns to solve the stochastic TOP with enough confidence, the learned parameters are shared with the other agents, therefore eliminating redundant learning to solve the single-agent OP.
In the second stage (fictitious-play stage), the agents update their parameters independently, and sequentially by playing against the opponents' ``stationary'' policies.

Algorithm~\ref{algo:ec-brpl} outlines the complete procedure.
We first initialize the algorithm with a hashmap $\Theta$ of random policy parameters of the agents, entropy-decaying step size $\delta h > 0$, an \fix{early-stopping entropy} $\mathcal{H}_{\mathrm{stop}}$, and an initial \textit{freezing point} (IFP) $\mathcal{H}^0$.
Further, $0 < \mathcal{H}_{\mathrm{stop}} < \mathcal{H}^0<\mathcal{H}_{\max}$, where $\mathcal{H}_{\mathrm{max}}$ is the highest possible entropy induced by the initial uniform random action selection, prize distribution, and the others' random behavior.
\fix{Ignoring the entropy contribution of the others, and the prizes, it holds that $\mathcal{H}_{\max} \leq \bar{\mathcal{H}}_{\max}$ where $\bar{\mathcal{H}}_{\max} = -\ln{\frac{1}{|\mathcal{V}|}}$, the entropy of the uniform random policy.}
Let $t_{\max}$ be the number of timesteps to train.
We use \texttt{step}$(a^t, s^t)$, and \texttt{update}$(\theta_j, r_j^t, o_j^{t-1}, a_j^{t}, o_j^{t})$ functions to evolve the environment by a single timestep, and to update the parameters of $\theta_j$ with stochastic gradient-descent.

At the beginning, $j=1$'s policy is updated while the others play randomly, until its entropy $\mathcal{H}[\pi_j^{\theta_j}]$ reaches the IFP that is $\mathcal{H}^0$ (\textbf{line 6}).
Then the parameters of the first policy are shared among the others, and the algorithm exits the bootstrapping stage (\textbf{line 8}).
Next, agent $j=2$ enters the fictitious play stage to update its policy playing against the updated policy of $j=1$, $\Theta[1]$, and the random policies of the others $\pi^{\theta^0}_{N\backslash [1,2]}$.
This process repeats until the last agents' policy reaches the IFP \textbf{(line 11)}.

\begin{algorithm}
\caption{FORL: Fictitious Ordinal Response Learning}
\label{algo:ec-brpl}
\KwIn{$\mathcal{H}^0$, $\delta h$, $\mathcal{H}_{\mathrm{stop}}$, $\Theta = \{i: \theta_i^0; \forall i \in \mathcal{N}\}$}
$\kappa(t)$, $j \gets 1$, $\mathcal{H}^{\kappa(t)} \gets \mathcal{H}^0$\tcp*{Initialization}
\While{$t \leq t_{\max}$}{
    $\theta_i \gets \Theta[i], \forall i \in \mathcal{N}$\;
    $\fix{a^t_i} \gets \fix{\pi_i^{\theta_i}[\cdot|O_i]}$, $\forall i \in \mathcal{N}$\;
    $s^t, r^t \gets \texttt{step}(a^{t}, s^{t-1})$\tcp*{Environment step}
    $\Theta[j] \gets \texttt{update}(\theta_j, r_j^t, o_j^{t-1}, a_j^{t}, o_j^{t})$\;
    \If{$\mathcal{H}[\pi_{j}^{\theta_j}] \fix{\leq} \mathcal{H}^{\kappa(t)}$}{
        \If{$j = 1 \land \kappa(t)=1$}{
            $\Theta[i] \gets \Theta[j]$ $\forall i \in \mathcal{N}$\tcp*{Bootstrapping}
        }
        $j \gets j + 1$; $\kappa(t) \gets t$\tcp*{Next policy}
        \If{$j \geq n$}{
            $\mathcal{H}^{\kappa(t)} \gets \mathcal{H}^{\kappa(t)} - \delta h$\tcp*{Reduce entropy}
            $j \gets 1$\tcp*{Reset $j$ for next round}
            \If{$\mathcal{H}^{\kappa(t)} < \mathcal{H}_{\mathrm{stop}}$}{
                \textbf{break}\tcp*{Exit while loop}
            }
        }
    }
    $t \gets t + 1$\;
}
\Return{$\Theta$}
\end{algorithm}
Then we reduce the entropy at the next freezing point by $\delta h$ amount \textbf{(line 12)}, and repeat the procedure.
The algorithm terminates when the last agents' policy reaches the prescribed early-stopping entropy $\mathcal{H}_{\mathrm{stop}}$ or till the total number of training timesteps $t_{\max}$ is reached \textbf{(line 14, 20)}.
Finally, the hashmap of the most recent policies are returned (\textbf{line 21}).
When the game does not possess a \fix{PNE}, the entropy of the policies is unlikely to reach 0, assuming $\mathcal{H}_{\mathrm{stop}}$ is set to 0.
The following proposition shows that Algorithm \ref{algo:ec-brpl} always terminates with a joint policy $\pi^*$ which is a PNE if one exists, and otherwise provides confidence bounds on the quality of the terminating policy.

\begin{proposition}
Suppose $t_{\max}$ is chosen sufficiently large and $\mathcal{H}_{stop} = 0$.  Then algorithm \ref{algo:ec-brpl} will terminate with a joint policy $\pi^* = \times_{i \in \mathcal{N}} \pi_i^{\theta_i^*}$ which satisfies
\begin{enumerate}[label=(\arabic*)]
    \item if the game has a PNE, policy $\pi^*$ is a PNE.
    \item otherwise, the confidence of the terminating policy $\pi^*$ satisfies
    \[\mathcal{H}(\pi_i^{\theta_i^*}) \in [\mathcal{H}^{\kappa(t_{\max} +1)}, \mathcal{H}^{\kappa(t_{\max})}], \forall i\in \mathcal N.\]
\end{enumerate}
\end{proposition}

\begin{proof}
\textbf{(Case 1)}. If the game has a PNE, denoted by $\pi^{\#} = \times_{i \in \mathcal{N}} \pi_i^{\theta^{\#}_i}$, then $\mathcal{H}[\pi^{\#}] = 0$ since it is a pure strategy. The algorithm will not terminate until reaching $\mathcal{H}_{\mathrm{stop}}=0$ for a large enough $t_{\max}$.
\textbf{(Case 2)}.
If there exist no PNEs, the $\mathcal{H}_{\mathrm{stop}}$ can not be reached, then the algorithm will terminate with $t=t_{\max}$ steps. From the sequential property of the algorithm, the relation
\[
\mathcal{H}^{\kappa(t+1)} \leq \mathcal{H}(\pi_j^{\theta_j^t}) \leq \mathcal{H}(\pi_i^{\theta_i^t}) \leq \mathcal{H}^{\kappa(t)},\]
holds for all $j < i$.
This implies that when the algorithm exits at $t = t_{\max}$, the confidence
\[
\mathcal{H}[\pi^{\theta_i^*}_i] \in [\mathcal{H}^{\kappa(t+1)}, \mathcal{H}^{\kappa(t)}],\quad \forall i.
\]
\end{proof}

Furthermore, as a direct consequence of the previous theorem, the following corollary shows that the hyperparameter $\mathcal{H}_{\mathrm{stop}}$ indicates the expected confidence level of the obtained policy $\pi^*$.

\begin{corollary}
Given a large enough $t_{\max}$, if $\mathcal{H}_{\mathrm{stop}}$ is set such that there exists a mixed NE policy $\tilde{\pi}^* = \times_{i \in \mathcal{N}} \pi_i^{\theta_i^*}$, and $\mathcal{H}_{\mathrm{stop}} \geq \mathcal{H}[\tilde{\pi}^*]$, then Algorithm 2 terminates with a terminating policy $\pi^*$ satisfying
\[
\mathcal{H}[\pi^*] \leq \mathcal{H}_{\mathrm{stop}}.
\]
\end{corollary}

Above, we use the \textit{maximum entropy principle} to guarantee that FORL continues until either the last agent's policy meets $\mathcal{H}_{\mathrm{stop}}$ or the maximum steps are reached.

It is also possible to set $\mathcal{H}^0$ to $\mathcal{H}_{\max}$.
In this case, the entropy of each policy must decay from $\mathcal{H}_{\max}$ to $\mathcal{H}_{\mathrm{stop}}$ without bootstrapping as opposed to decaying from $\mathcal{H}^0$ to $\mathcal{H}_{\mathrm{stop}}$ for all agents except for the first one.
By setting $\mathcal{H}^{0}$ smaller than $\mathcal{H}_{\max}$, only the first agent's policy needs updating until $\mathcal{H}^{0}$, due to the bootstrapping.
\fix{In the experiments, we used $\mathcal{H}^{0}$ between $0.8 \mathcal{H}_{\max}$-$0.6 \mathcal{H}_{\max}$.}
Further, due to the sequential-nature of FORL, the experiences of the other agents except for the one who is being trained are discarded; a pretext for poor sample efficiency, especially with large batches.
However, compared to IPL, and parameter-sharing, FORL offered better training throughput with medium-sized batches ($\sim 2500$).

\subsection{Ordinal Rank Conditioning for Parameter-Sharing}
We augment the agent's policy $\pi_i^{\theta_i}$ by incorporating the OR into its observation space.
\fix{Specifically, instead of conditioning the policy solely on the local observation $O_i$, the policy is now conditioned on the augmented input $(O_i,\mathcal{I}_i)$.}
\fix{Thus, policy used here reads},
\begin{equation}
\pi_i^{\theta_i}\Big[A_i\mid O_i, \mathcal{I}_i \Big].
\label{eq:local_obs}
\end{equation}
The OR conditioning introduces ``state-aliasing'' to the learning process thus, preventing the policies overfitting to global ranks: the same global rank may map to different ORs depending on one's neighborhood at each stage, thereby increasing the variance, and providing better interpretability for the accumulated experiences.
This conditioning provides a particular advantage when the policy parameters are shared between the agents, i.e., $(\theta_1 = \theta_2 = \dots = \theta_n)$: we observe that these policies generalize much better to large team sizes compared to those conditioned on the global rank.

\noindent \textbf{Note on the scalability}: Even when a policy conditioned on the OR was deployed on an agent with a global rank that was absent in the training, the policy may still have encountered its OR during training, and can thus find ordinal best responses to its immediate opponents.

\subsection{Action Masking and Policy Modeling}

The observation space of an agent $i$ used for training contains its local state $\mathcal{S}_i$, the prize state $\mathbf{p}^t$, ordinal rank $\mathcal{I}_i(t)$, and an \textit{action mask} $\mathbf{m}_i^t \in \{1,0\}^{|\mathcal{V}|}$ where
\begin{equation*}
    \mathbf{m}_i^t[u] =
    \begin{cases}
        0; \text{ if } (x_i^t, u) \notin \mathcal{E}, \\
        1; \text{ otherwise.}
    \end{cases}
\end{equation*}
Here, $\mathbf{m}_i^t[u]$ denotes the mask corresponding to the $u$-th node $u \in \mathcal{V}$.
Then the policy outputs corresponding to masked nodes are manually set to 0, and renormalized: $\pi_i^{\theta_i}(a_i = u|o_i^t) \gets \frac{\pi_i^{\theta_i}(a_i^t = u|o_i^t) \odot \mathbf{m}^t_i[u]} {\sum_{u} \pi_i^{\theta_i}(a_i^t = u|o_i^t)}$ for all $u \in \mathcal{V}$, and `$\odot$' is the element-wise product between two vectors.
Masking however, can result in a lower $\mathcal{H}_{\max}$ compared to the non-masked case which is equivalent to the entropy of the uniform random policy with support $[0,|\mathcal{V}|]$, as the actions corresponding to non-neighboring nodes are explicitly ignored.
Therefore, a safer method for setting the $\mathcal{H}^{0}$ is to use the empirical maximum entropy of the game at the beginning of the training process.

We adopt a Transformer-based neural network architecture, TRXL-I (see~\cite{parisotto_stabilizing_2019} for details), to model the policies.
Each policy comprises 6 transformer units that have been shared between the policy and the value network.
Each transformer unit has 6 attention heads of dimensionality of 128, and an attention memory of 10 observations.
The output of the transformer sequence is passed to two linear layers in parallel: one for estimating the value-function output, and the second for computing the policy outputs.

\section{Experiments and Results}

\begin{figure}[t]
  \centering
  \begin{subfigure}[t]{0.65\columnwidth}
    \includegraphics[width=\linewidth,trim={0cm 11cm 0cm 0cm},clip]{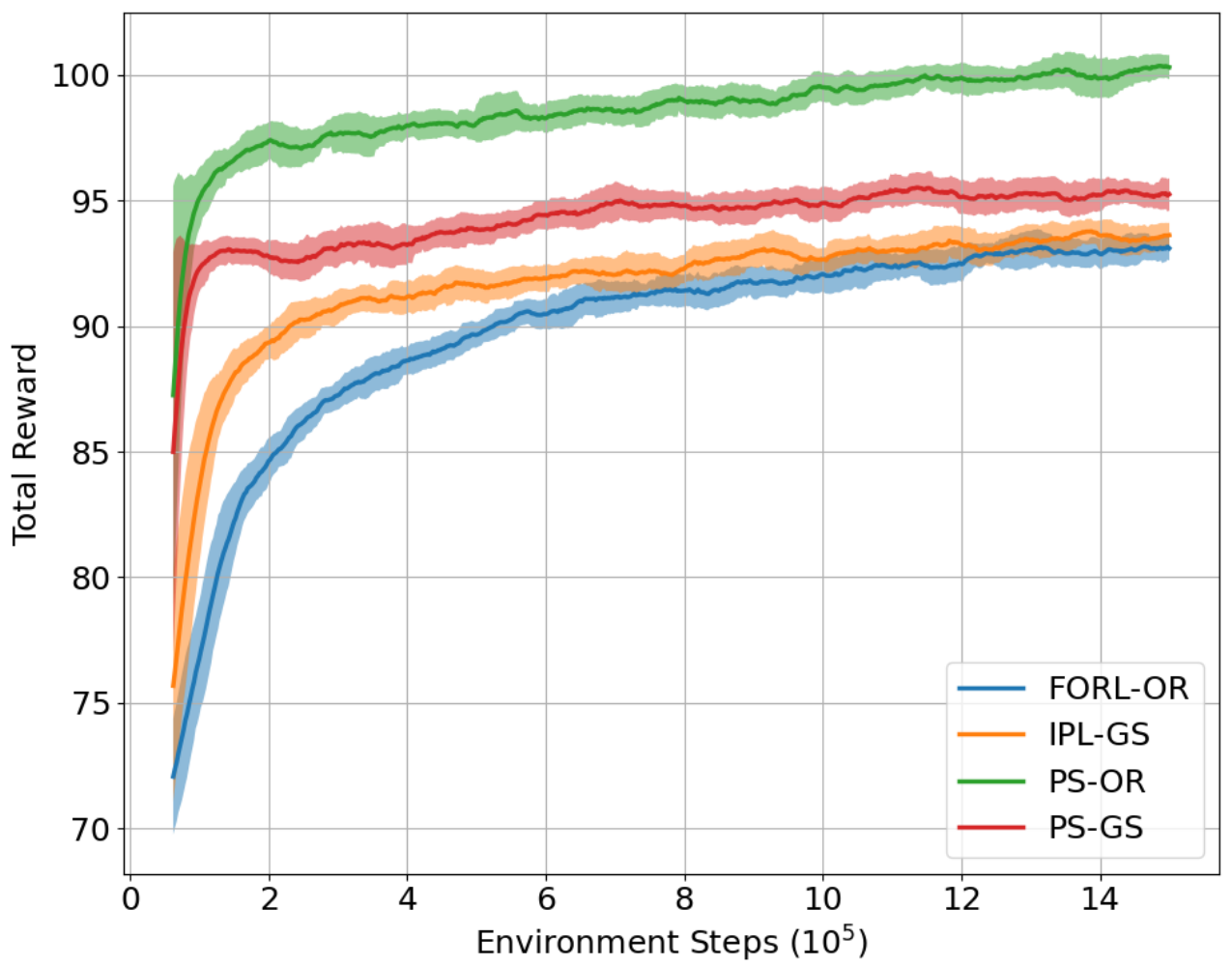}
    \caption{Training}
  \end{subfigure}
  \hfill
  \begin{subfigure}[t]{0.31\columnwidth}
    \includegraphics[width=\linewidth,trim={0cm -1cm 4cm 0cm},clip]{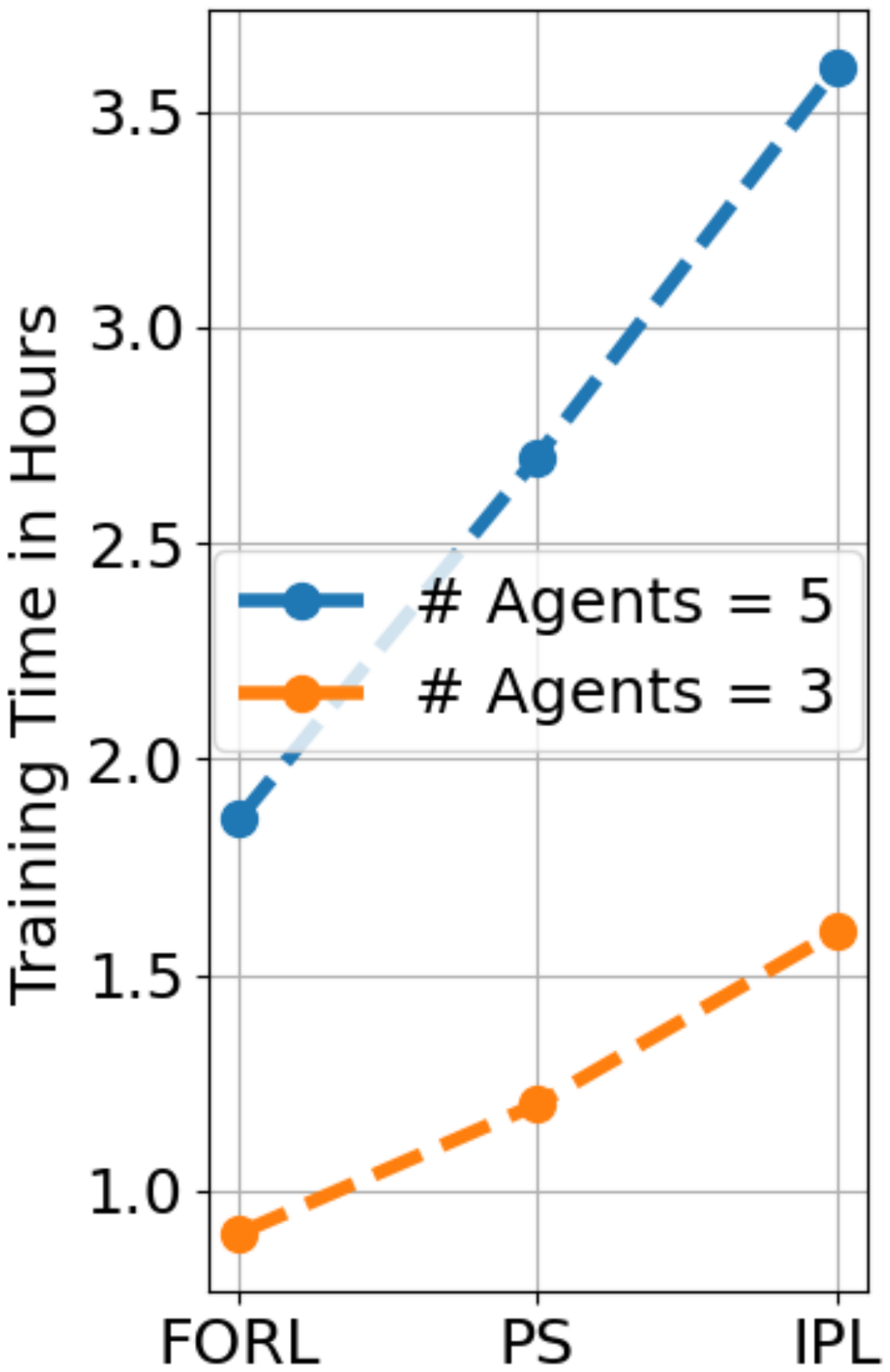}
    \caption{Time}
  \end{subfigure}
  \caption{\small Convergence (a) and training time (b) for 5 agents on the Manhattan graph with dynamic prizes. OR/GS denote Ordinal Rank and Global State conditioning.}
  \label{fig:convergence}
\end{figure}

\subsection{Simulation Environment}

We implemented the SPCOG on a custom Gym environment \cite{terry2021pettingzoogymmultiagentreinforcement}, trained on Ray RLlib \cite{liang2017ray} for improving the sampling process via multi-agent parallel execution.
We evaluated the proposed approach on real-world road-networks defined on a \SI{500}{\meter} $\times$ \SI{500}{\meter} area of Stockholm, Manhattan road networks under static, and dynamic prizes.
Each road network graph consists of 20 and 50 nodes, where each \textit{dead-end} node served as a terminal node $d$ with equal payoffs $p_d$.
The two graphs were chosen due to their unique ``ring'', and ``grid'' like structures. See Fig.~\ref{fig:cities} at the end.
During an episode, the agents were initialized randomly, i.e., each $s_i, \forall i$ were chosen uniformly from non-terminal nodes.

The prizes of non-terminal nodes were drawn from $\mathbf{p}^t \sim \times_{u \in \mathcal{V}} \texttt{Uniform}_{u}(0,10)$ whereas each $p_d$ were set to 15.
The policies were updated using the Proximal Policy Optimization (PPO) algorithm with batches of 2500 observations, mini-batches of 200, and 10 stochastic gradient descent iterations per update.

\subsection{Training Results}

Fig.~\ref{fig:convergence}-(a) shows the training results for IPL, FORL, and parameter sharing (PS) methods in the Manhattan environment with dynamic prizes for a team of 5 agents.
The policies trained using IPL were provided with the global-state vector (IPL-GS), FORL with the local observations conditioned on the OR according to Eq.~\eqref{eq:local_obs} (FORL-OR), PS using both global and local information vectors, denoted by PS-GS, PS-OR, respectively.
The OR-conditioned policies trained using FORL converged to almost the same total reward profile as IPL-GS, and PS-GS: thus, indicating that OR can indeed successfully encode the global state in prize-collecting games, and serve as a strong inductive bias when global state observation is not feasible.

The parameter shared policies conditioned on the OR (PS-OR), significantly outperformed PS-GS, and also the shared policies conditioned on the global rank.
While we do not show the training curve of the latter for clarity, we provide a quantitative comparison between the two on their generalizability in Fig.~\ref{fig:generalizability}-(b).
Generally, training methods where the parameters of the policy or the value function are shared among the agents tend to score higher total rewards compared to independent learning methods such as FORL, IPL, due to agents being able to reason about the others' policies symmetrically.
In addition to that, the state-aliasing induced by the OR conditioning allows the shared policies to collect more experiences for the same OR across the agents leading to more exploration, and thus, better performances.

\begin{figure}[t]
    \centering
    \begin{subfigure}[t]{0.48\columnwidth}
      \includegraphics[width=\linewidth,trim={0cm 9.5cm 0cm 0cm},clip]{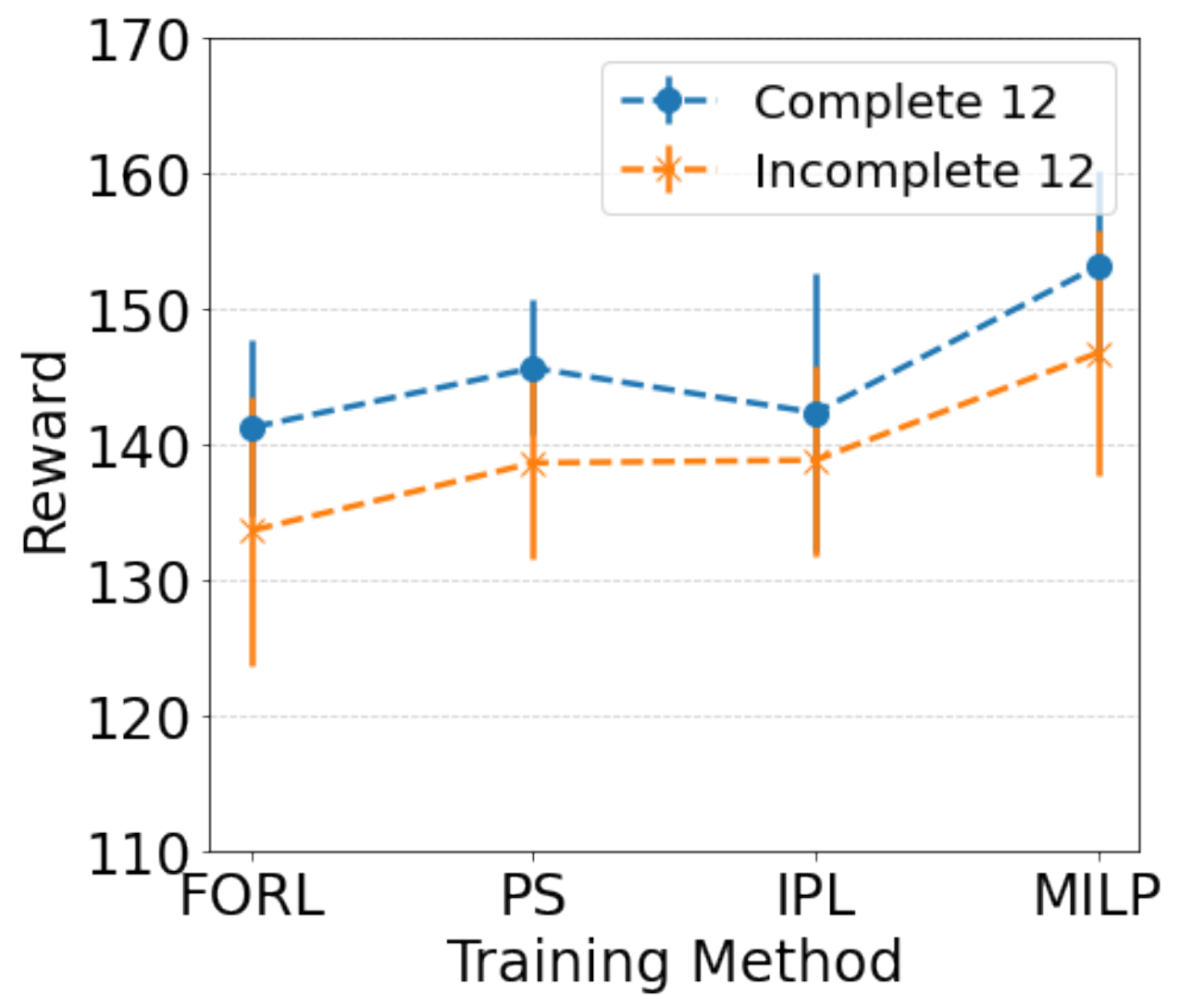}
      \caption{Optimality gap}
    \end{subfigure}
    \hfill
    \begin{subfigure}[t]{0.48\columnwidth}
      \includegraphics[width=\linewidth,trim={0cm 9.5cm 0cm 0cm},clip]{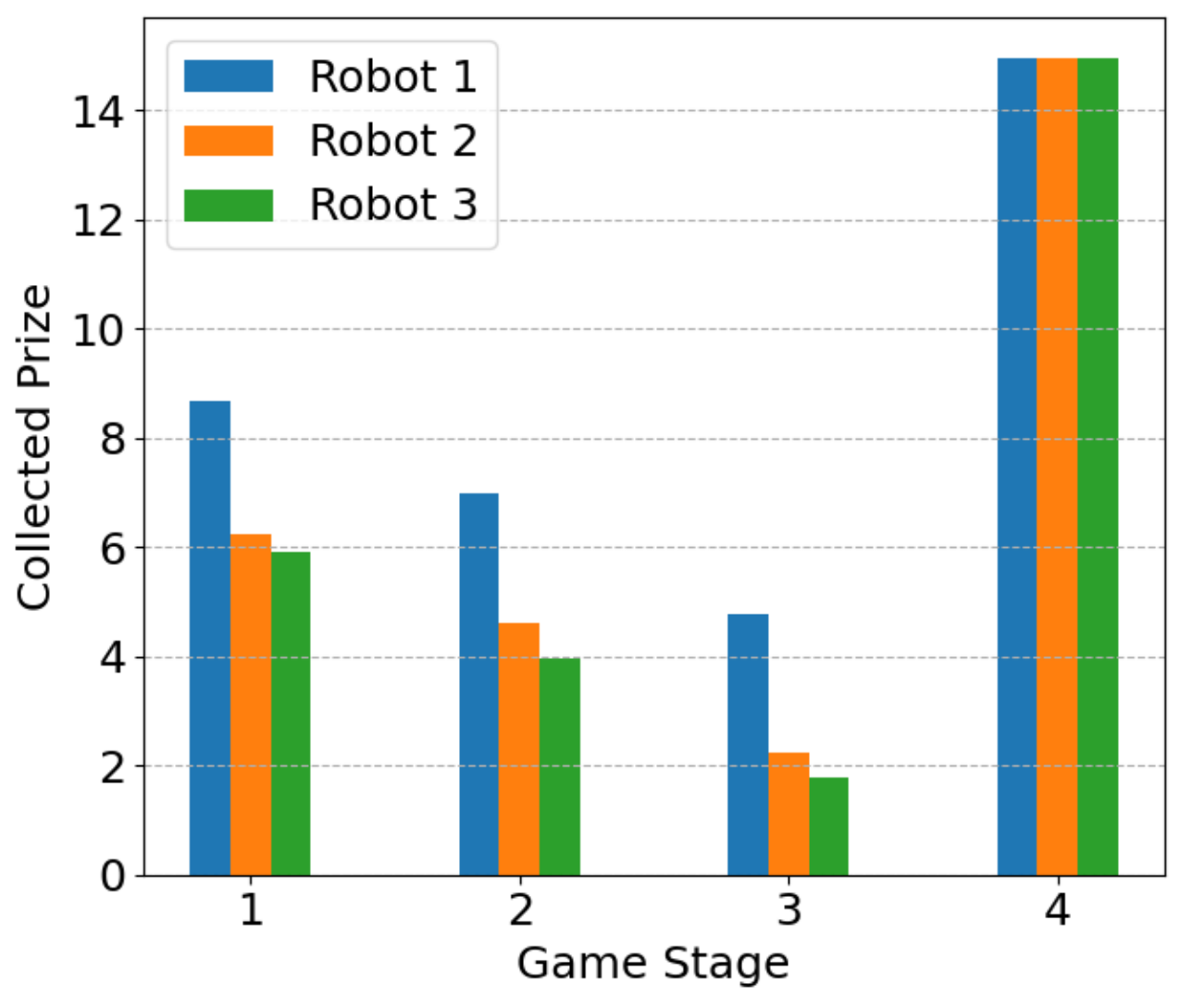}
      \caption{Stage rewards}
    \end{subfigure}
    \caption{\small (a) Optimality gap between SPCOG and the equivalent TOP. (b) Per-stage prizes in a complete graph under FORL ($L=4$, first 3 agents shown).}
    \label{fig:stationary}
\end{figure}

Fig.~\ref{fig:convergence}-(b) shows the training time for each method.
Noticeably, our method takes much less training time to reach approximately the same total team-reward profile as PS-GS, and IPL, with speed-up factors 1.56 and 2.2, to train for the same number of environment steps, and gradient descent iterations.

In Fig.~\ref{fig:stationary}, we report total rewards in SPCOGs played on two graphs (complete and incomplete random graphs that have 12 nodes) with stationary prizes for each training method, and compared them to an equivalent TOP whose solution was obtained by solving a mixed-integer linear program (MILP) as presented in~\cite{gunawan_orienteering_2016} using Gurobi.
In the complete graph, PS-OR reports the highest total reward that is $95\%$ the optimal solution.
In the incomplete graph, the worst case equilibria was obtained with FORL with a total reward that is 87\% of the optimal solution.

Fig.~\ref{fig:stationary}-(b) shows the prizes collected by the agents at each stage in a game of 4 stages ($L = 4$), on a complete graph (only the prizes of the first three agents trained using FORL are shown).
The PNE of the game states that the prizes collected at each stage obey the agents' ranks.
Similarly, we observe that the senior-most agent (robot 1) collects the highest prize at each stage, while robot 2 and robot 3 pick the second and third largest prizes.
In the final stage, all the agents reach the terminal node receiving the terminal reward.

\subsection{Generalizability and Scalability}

We evaluated the policies' generalizability by \textit{zero-shot} transferring them onto an imbalanced prize distribution that resembles a demand-driven taxi-fare distribution in a city, i.e.,$p^t_{u} \sim \mathrm{Normal}(\beta_u, \sigma)$, where $\beta_u$ is the mean prize at a node that is inversely proportional to the proximity to the map center, whereas $\sigma$ is the standard deviation.

\begin{figure}[t]
  \centering
  \begin{subfigure}[t]{0.48\columnwidth}
    \includegraphics[width=\linewidth,trim={0cm 9.5cm 0cm 0cm},clip]{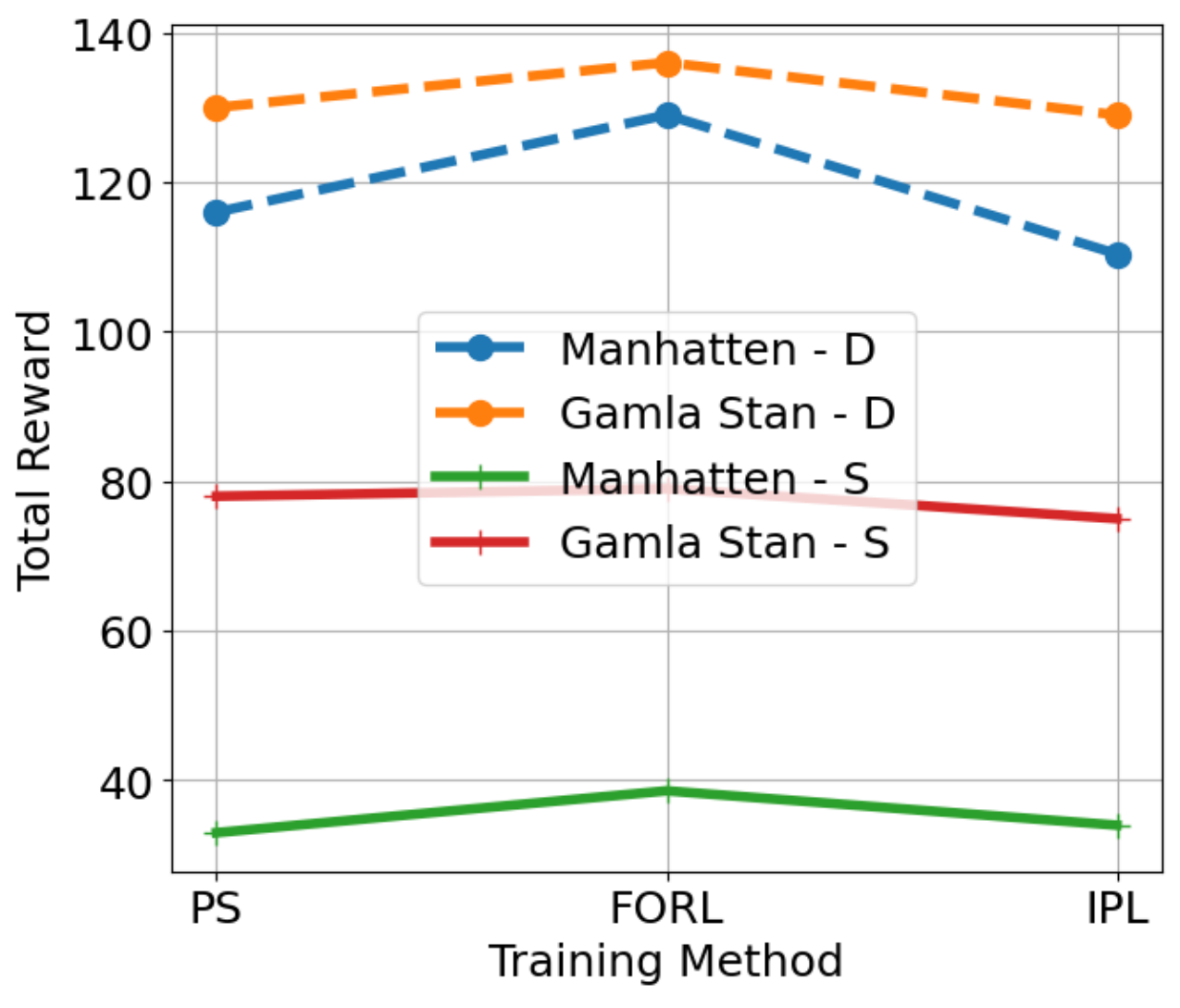}
    \caption{Zero-shot transfer}
  \end{subfigure}
  \hfill
  \begin{subfigure}[t]{0.48\columnwidth}
    \includegraphics[width=\linewidth,trim={0cm 9.5cm 0cm 0cm},clip]{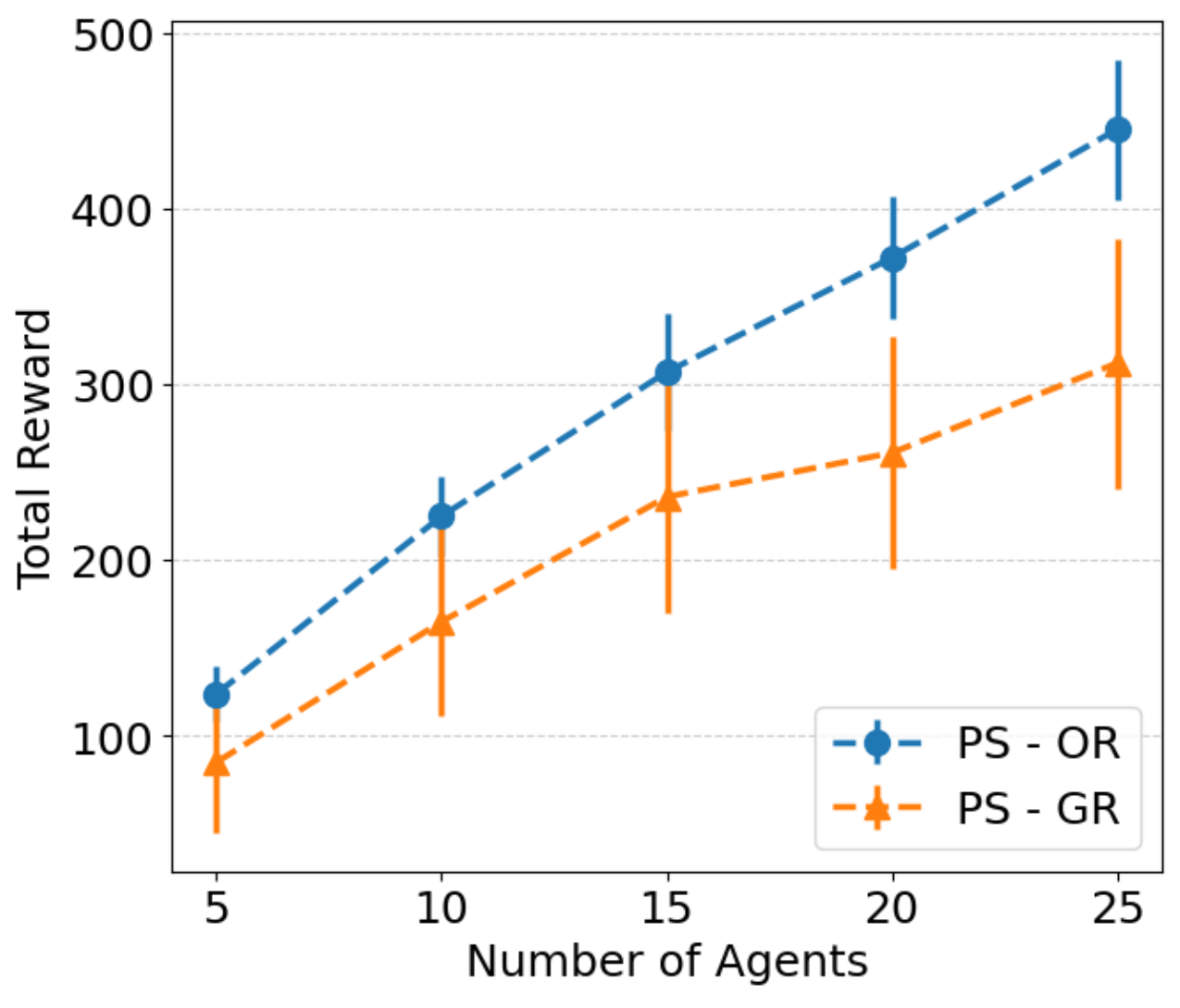}
    \caption{Scalability}
  \end{subfigure}
  \caption{\small (a) Zero-shot transfer to region-specific normal prizes (D: dynamic, S: stationary). (b) Scalability of PS-OR vs.\ PS-GR (global rank) up to 25 agents.}
  \label{fig:generalizability}
\end{figure}

Fig.~\ref{fig:cities} shows the joint distribution of the prizes in the two city graphs partitioned into four and five price zones.
The policies trained using FORL outperformed both PS-OR, and IPL-GS in generalizing to imbalanced prize distributions in both the cities, under both the stationary and dynamic settings (see Fig.~\ref{fig:generalizability}-a).
Finally, we evaluated the scalability of the shared policies to different team sizes on the Manhattan environment with dynamic prizes.
Both of the considered policies were trained using local information, but, conditioned on the ordinal rank (PS-OR), and the global rank (PS-GR).
The total reward under PS-OR scaled linearly up to teams of 25 agents with only marginal increment in the variance, whereas PS-GR showed poor, sublinear, total-reward gains with much higher variance for the same team sizes (see Fig.~\ref{fig:generalizability}-b).

\section{Conclusion}
We introduced SPCOG, a Markov game formulation of the TOP for planning in self-interested multi-agent teams operating in stochastic, graph-structured environments.
A theoretical analysis conducted on complete graphs shows that there exists a PNE that coincides with the optimal routing solution under a rank-based conflict resolution.
We propose the concept of OR --the effective rank of an agent within a localized neighborhood to choose local best responses at each stage of the game.
The state aliasing induced by the OR in parameter-sharing MARL also allows learning policies that significantly outperform those conditioned on the global ranks.
Finally, we propose a fictitious play-inspired entropy regulated training scheme ``FORL" to find convergent policies on general graphs.
We show the learned policies achieve great scalability when conditioned on the OR, and an optimality gap of 87\%--95\% compared to MILP solution of an equivalent TOP.

\bibliographystyle{IEEEtran}
\bibliography{root}

\begin{figure}[t]
\centering
\begin{subfigure}[t]{0.48\columnwidth}
  \includegraphics[width=\linewidth,trim={0cm 6cm 1cm 1cm},clip]{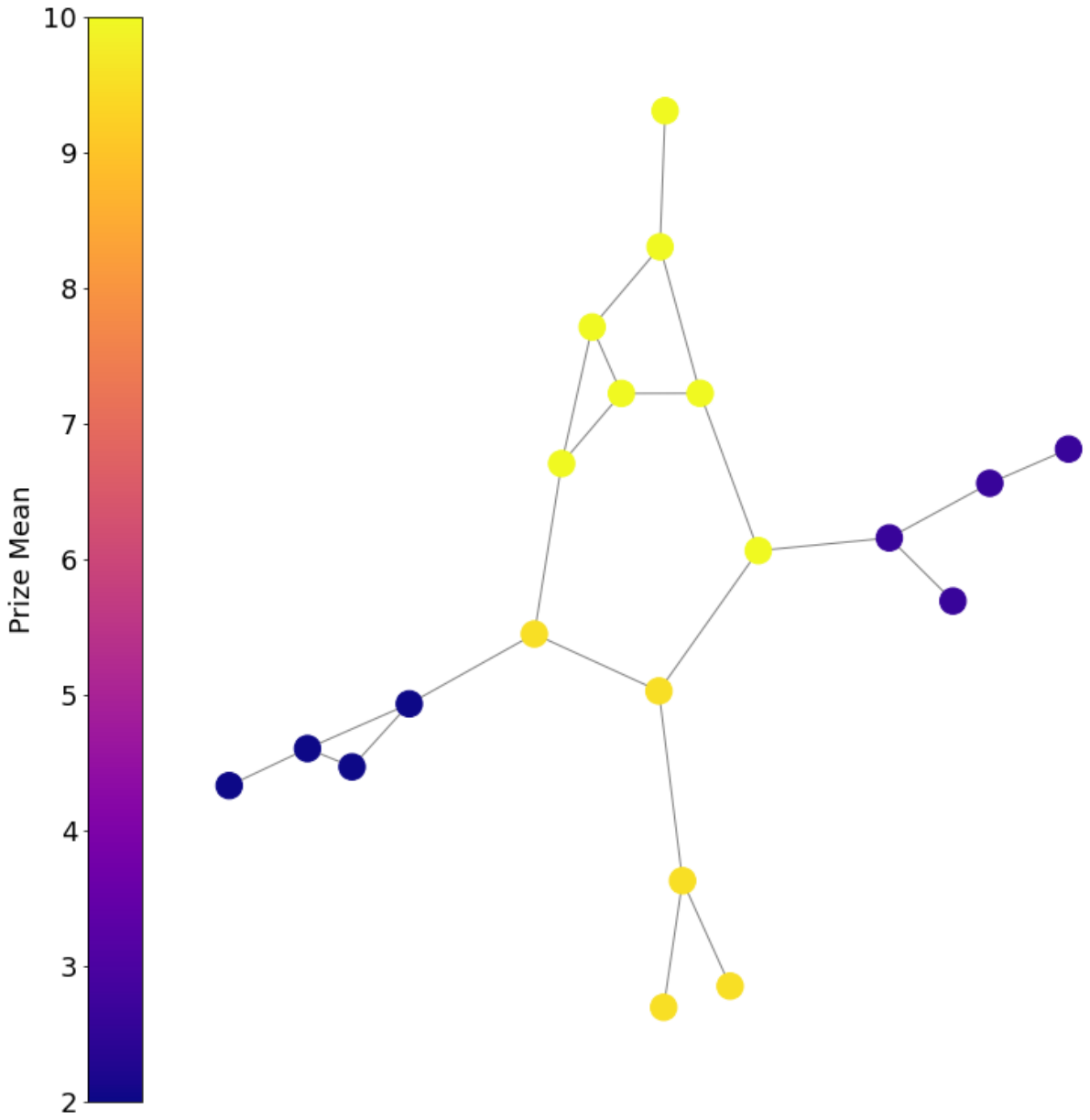}
  \caption{Stockholm}
\end{subfigure}
\hfill
\begin{subfigure}[t]{0.48\columnwidth}
  \includegraphics[width=\linewidth,trim={1cm 6cm 1cm 0cm},clip]{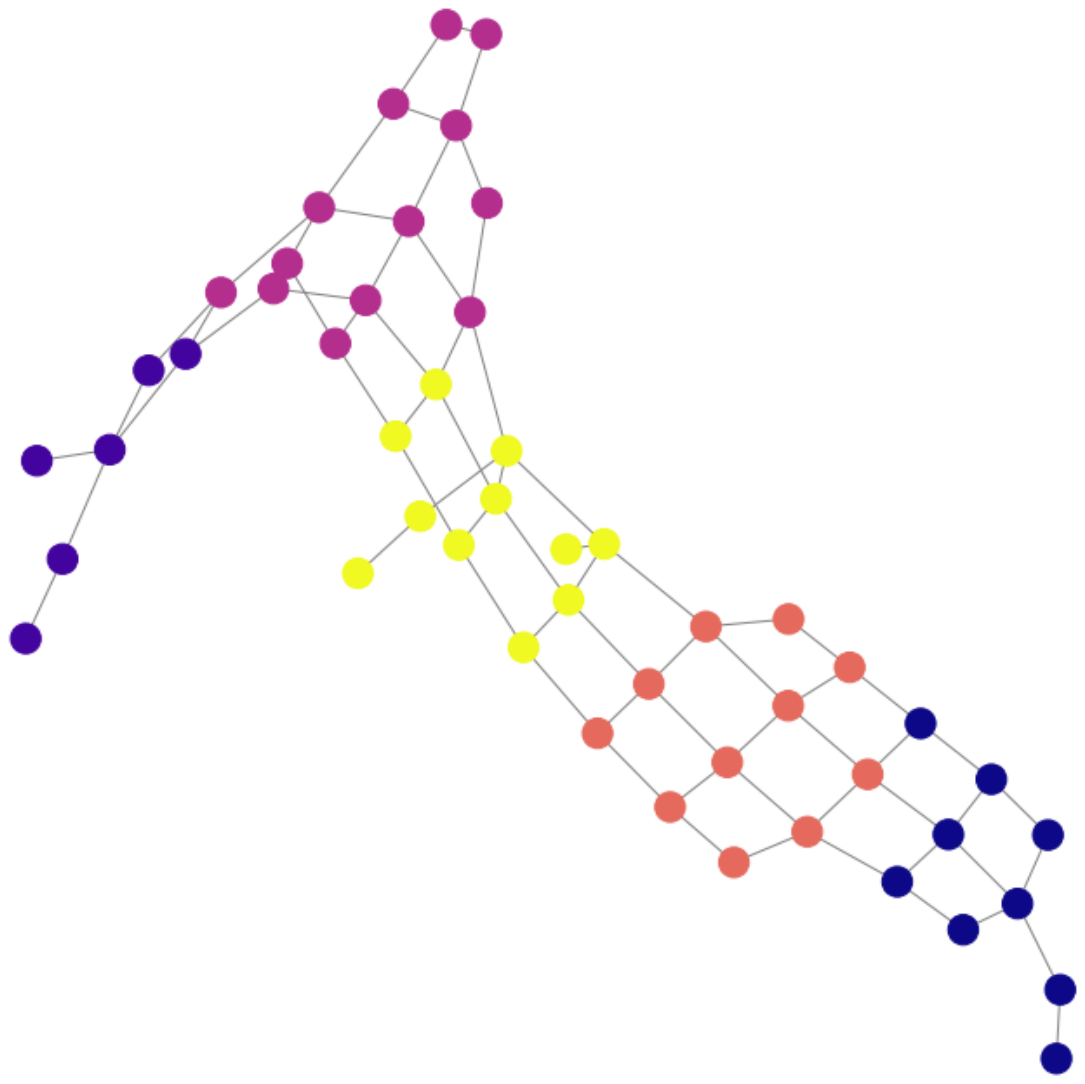}
  \caption{Manhattan}
\end{subfigure}
\caption{Two different road graphs, Stockholm (left) and Manhattan (right) used for the training. Node color encodes mean prize $\bar{p}_u \propto 1/\norm{\mathrm{pos}(u) - \mathrm{pos}(\bar{u})}$, where $\bar{u}$ is the map center.}
\label{fig:cities}
\end{figure}

\end{document}